%% file: main.tex
\RequirePackage[l2tabu,orthodox]{nag}
\documentclass
[letterpaper,11pt,]
{article}

\usepackage{etex}
\usepackage{verbatim}
\usepackage{xspace,enumerate}
\usepackage[dvipsnames]{xcolor}
\usepackage[T1]{fontenc}
\usepackage[full]{textcomp}
\usepackage[american]{babel}
\usepackage{mathtools}
\usepackage{amsthm}
\usepackage[
letterpaper,
top=1in,
bottom=1in,
left=1in,
right=1in]{geometry}
\usepackage{newpxtext} %
\usepackage{textcomp} %
\usepackage[varg,bigdelims]{newpxmath}
\usepackage[scr=rsfso]{mathalfa}%
\usepackage{bm} %
\let\mathbb\varmathbb
\usepackage{microtype}
\usepackage[pagebackref,colorlinks=true,urlcolor=blue,linkcolor=blue,citecolor=OliveGreen]{hyperref}
\usepackage[capitalise,nameinlink]{cleveref}
\crefname{lemma}{Lemma}{Lemmas}
\crefname{fact}{Fact}{Facts}
\crefname{theorem}{Theorem}{Theorems}
\crefname{corollary}{Corollary}{Corollaries}
\crefname{claim}{Claim}{Claims}
\crefname{example}{Example}{Examples}
\crefname{algorithm}{Algorithm}{Algorithms}
\crefname{problem}{Problem}{Problems}
\crefname{definition}{Definition}{Definitions}
\crefname{exercise}{Exercise}{Exercises}
\crefname{condition}{Condition}{Conditions}
\usepackage{amsthm}

\newtheorem{theorem}{Theorem}[section]
\newtheorem*{theorem*}{Theorem}
\newtheorem{lemma}[theorem]{Lemma}
\newtheorem*{lemma*}{Lemma}
\newtheorem{fact}[theorem]{Fact}
\newtheorem*{fact*}{Fact}
\newtheorem{proposition}[theorem]{Proposition}
\newtheorem*{proposition*}{Proposition}

\newtheorem*{corollary*}{Corollary}

\newtheorem*{hypothesis*}{Hypothesis}

\newtheorem*{conjecture*}{Conjecture}
\theoremstyle{definition}
\newtheorem{definition}[theorem]{Definition}
\newtheorem*{definition*}{Definition}

\newtheorem*{construction*}{Construction}

\newtheorem*{example*}{Example}

\newtheorem*{question*}{Question}
\newtheorem{algorithm}[theorem]{Algorithm}
\newtheorem*{algorithm*}{Algorithm}

\newtheorem*{assumption*}{Assumption}

\newtheorem*{problem*}{Problem}

\newtheorem*{openquestion*}{Open Question}
\theoremstyle{remark}

\newtheorem*{claim*}{Claim}

\newtheorem*{remark*}{Remark}

\newtheorem*{observation*}{Observation}
\usepackage{paralist}
\let\originalleft\left
\let\originalright\right
\renewcommand{\left}{\mathopen{}\mathclose\bgroup\originalleft}
\renewcommand{\right}{\aftergroup\egroup\originalright}
\usepackage{turnstile}
\usepackage{mdframed}
\usepackage{tikz}
\usetikzlibrary{positioning}
\usepackage{caption}
\DeclareCaptionType{Algorithm}
\usepackage{newfloat}
\usepackage{array}
\usepackage{subfig}
\usepackage{bbm}
\usepackage{xparse}
\usepackage{amsthm} %
\makeatletter
\let\latexparagraph\paragraph
\RenewDocumentCommand{\paragraph}{som}{%
  \IfBooleanTF{#1}
    {\latexparagraph*{#3}}
    {\IfNoValueTF{#2}
       {\latexparagraph{\maybe@addperiod{#3}}}
       {\latexparagraph[#2]{\maybe@addperiod{#3}}}%
  }%
}
\newcommand{\maybe@addperiod}[1]{%
  #1\@addpunct{.}%
}
\makeatother

\newcommand{\Authornote}[2]{}
\newcommand{\Authornotecolored}[3]{}
\newcommand{\Authorcomment}[2]{}
\newcommand{\Authorfnote}[2]{}

\usepackage{boxedminipage}

\newcommand{\Paren}[1]{\left(#1\right)}

\newcommand{\brac}[1]{[#1]}
\newcommand{\Brac}[1]{\left[#1\right]}

\newcommand{\abs}[1]{\lvert#1\rvert}
\newcommand{\Abs}[1]{\left\lvert#1\right\rvert}

\newcommand{\set}[1]{\{#1\}}
\newcommand{\Set}[1]{\left\{#1\right\}}

\newcommand{\norm}[1]{\lVert#1\rVert}

\newcommand{\normo}[1]{\norm{#1}_1}

\newcommand{\iprod}[1]{\langle#1\rangle}

\newcommand{\Esymb}{\mathbb{E}}
\newcommand{\Psymb}{\mathbb{P}}

\DeclareMathOperator*{\E}{\Esymb}

\newcommand{\suchthat}{\;\middle\vert\;}

\newcommand\bdot\bullet

\DeclareMathOperator{\OPT}{OPT}

\DeclareMathOperator{\poly}{poly}

\DeclareMathOperator{\sign}{sign}

\newcommand{\Z}{\mathbb Z}
\newcommand{\N}{\mathbb N}
\newcommand{\R}{\mathbb R}

\renewcommand{\leq}{\leqslant}
\renewcommand{\le}{\leqslant}
\renewcommand{\geq}{\geqslant}

\let\epsilon=\varepsilon
\numberwithin{equation}{section}
\newcommand\MYcurrentlabel{xxx}
\newcommand{\MYstore}[2]{%
  \global\expandafter \def \csname MYMEMORY #1 \endcsname{#2}%
}
\newcommand{\MYload}[1]{%
  \csname MYMEMORY #1 \endcsname%
}
\newcommand{\MYnewlabel}[1]{%
  \renewcommand\MYcurrentlabel{#1}%
  \MYoldlabel{#1}%
}
\newcommand{\MYdummylabel}[1]{}
\newcommand{\torestate}[1]{%
  \let\MYoldlabel\label%
  \let\label\MYnewlabel%
  #1%
  \MYstore{\MYcurrentlabel}{#1}%
  \let\label\MYoldlabel%
}
\newcommand{\restatetheorem}[1]{%
  \let\MYoldlabel\label
  \let\label\MYdummylabel
  \begin{theorem*}[Restatement of \cref{#1}]
    \MYload{#1}
  \end{theorem*}
  \let\label\MYoldlabel
}
\newcommand{\restatelemma}[1]{%
  \let\MYoldlabel\label
  \let\label\MYdummylabel
  \begin{lemma*}[Restatement of \cref{#1}]
    \MYload{#1}
  \end{lemma*}
  \let\label\MYoldlabel
}
\newcommand{\restateprop}[1]{%
  \let\MYoldlabel\label
  \let\label\MYdummylabel
  \begin{proposition*}[Restatement of \cref{#1}]
    \MYload{#1}
  \end{proposition*}
  \let\label\MYoldlabel
}
\newcommand{\restatefact}[1]{%
  \let\MYoldlabel\label
  \let\label\MYdummylabel
  \begin{fact*}[Restatement of \cref{#1}]
    \MYload{#1}
  \end{fact*}
  \let\label\MYoldlabel
}
\newcommand{\restate}[1]{%
  \let\MYoldlabel\label
  \let\label\MYdummylabel
  \MYload{#1}
  \let\label\MYoldlabel
}

\newcommand{\e}{\epsilon}

\allowdisplaybreaks
\newcommand*{\Id}{\mathrm{Id}}

\newcommand{\ind}[1]{\mathbf{1}_{\Brac{#1}}}

\newcommand{\err}[1]{\mathrm{err}(#1)}

\title{
  Optimal SQ Lower Bounds for Learning Halfspaces with Massart Noise\thanks{This project has received funding from the European Research Council (ERC) under the European Union’s Horizon 2020 research and innovation programme (grant agreement No 815464).}
}

\author{
  Rajai Nasser\thanks{ETH Z\"urich.}
  \and
  Stefan Tiegel\footnotemark[2]
}

\begin{document}

\pagestyle{empty}

\maketitle
\thispagestyle{empty} %

\begin{abstract}

\input{content/abstract}

\end{abstract}

\clearpage

\microtypesetup{protrusion=false}
\tableofcontents{}
\microtypesetup{protrusion=true}

\clearpage

\pagestyle{plain}
\setcounter{page}{1}

\input{content/introduction}
\input{content/techniques}
\input{content/preliminaries}
\input{content/hardness_result}

\input{content/hard_distributions}

\newpage

\phantomsection
\addcontentsline{toc}{section}{References}
\bibliographystyle{amsalpha}
\bibliography{bib/custom}

\appendix

\input{content/generic_sq_lower_bound.tex}
\input{content/moment_bound.tex}

\end{document}

%% file: content/abstract.tex
We give tight statistical query (SQ) lower bounds for learnining halfspaces in the presence of Massart noise.
In particular, suppose that all labels are corrupted with probability at most $\eta$.
We show that for arbitrary $\eta \in [0,1/2]$ every SQ algorithm achieving misclassification error better than $\eta$ requires queries of superpolynomial accuracy or at least a superpolynomial number of queries.
Further, this continues to hold even if the information-theoretically optimal error $\OPT$ is as small as $\exp\Paren{-\log^c(d)}$, where $d$ is the dimension and $0 < c < 1$ is an arbitrary absolute constant, and an overwhelming fraction of examples are noiseless.
Our lower bound matches known polynomial time algorithms, which are also implementable in the SQ framework.
Previously, such lower bounds only ruled out algorithms achieving error $\OPT + \e$ or error better than $\Omega(\eta)$ or, if $\eta$ is close to $1/2$, error $\eta - o_\eta(1)$, where the term $o_\eta(1)$ is constant in $d$ but going to 0 for $\eta$ approaching $1/2$.

As a consequence, we also show that achieving misclassification error better than $1/2$ in the $(A,\alpha)$-Tsybakov model is SQ-hard for $A$ constant and $\alpha$ bounded away from 1.

%% file: content/introduction.tex
\section{Introduction}
\label{sec:introduction}

Arguably one of the most fundamental problems in the area of machine learning and learning theory, going back to the Perceptron Algorithm~\cite{R58}, is the problem of learning halfspaces, or Linear Threshold Functions (LTFs):
Fix $w \in  \R^M$ and $\theta \in \R$, an LTF is a function $f \colon \R^M \rightarrow \Set{-1,1}$ such that $f(x) = 1$ if $\iprod{w,x} \geq \theta$ and $-1$ otherwise.
The associated learning problem is as follows: We observe samples $(x,y)$ where $x\in\R^M$ is drawn from a fixed but unknown distribution $D_x$ and $y\in\{-1,+1\}$ is, a possibly noisy version of, $f(x)$.
We call $x$ the \emph{example} and $y$ the \emph{label}.
Let $D$ denote the joint distribution of $(x,y)$, the goal is to output a hypothesis $h$ such that the \emph{misclassification error} $$\err{h} \coloneqq \Psymb_{(x,y) \sim D} \brac{h(x) \neq y}$$ is minimized.
For the purpose of this paper we consider the case where $\theta = 0$.
In this work we make progress on a central question in the field: Identifying under which types of noise achieving small misclassification error is possible.
On a conceptual level, we show that already as soon as only very few of the labels are flipped with some probability $\eta$, it is likely to be computationally hard to achieve error better than $\eta$.
Even if the optimal error is much smaller than this.

\paragraph{Realizable Case, Random Classification Noise, and Agnostic Model}
In the noiseless case, also called \emph{realizable case}, it holds that $y = f(x)$ for all $x$.
In this setting it is well-known that linear programming can achieve misclassfication error at most $\e$ efficiently, i.e., in time polynomial in $M$ and $\frac{1}{\e}$, and reliably, i.e., with probability close to 1, corresponding to Valiant's PAC model~\cite{V84}.
When considering noisy labels, the two most well-studied models are \emph{Random Classification Noise} (RCN)~\cite{AL88} and the \emph{agnostic model}~\cite{H92,KSS94}.
In the former each sample $(x,y)$ is generated by first drawing $x \sim D_x$ and then setting $y = f(x)$ with probability $1-\eta$ and setting $y = -f(x)$ with probability $\eta$ for some $\eta \in (0,1)\setminus \set{1/2}$.
It can be shown that in this model the information-theoretic optimal misclassification error is $\eta$ and it is known how to efficiently find an LTF achieving misclassification error arbitrarily close to this~\cite{BFKV98}.
However, one clear drawback is that the assumption that the magnitude of the noise is uniform across all examples is unrealistic.
On the other extreme, in the agnostic model, no assumption whatsoever is placed on the joint distribution $D$.
It is now believed that it is computationally hard to output any hypothesis that achieves error even slightly better than $1/2$.
This holds even when the information-theoretic misclassification error is a function going to zero when the ambient dimension goes to infinity~\cite{D16}.
This is based on a hardness reduction to a problem widely believed to be computationally intractable.

\paragraph{A More Realistic Yet Computationally Tractable Noise Model}
Given the above results a natural question to ask is whether there exists a more realistic noise model in which it is still computationally tractable to achieve non-trivial guarantees.
A promising candidate is the so-called \emph{Massart noise model} which is defined as follows

\begin{definition}
    \label{def:massart_noise}
    Let $D_x$ be a distribution over $\R^M$ and let $f \colon \R^M \rightarrow \set{-1,1}$ be an LTF.
    For $\eta \in [0,1/2]$, we say that a distribution $D$ over $\R^M \times \set{-1,1}$ satisfies the \emph{$\eta$-Massart noise condition} with respect to the \emph{hypothesis} $f$ and to the \emph{marginal distribution} $D_x$ if there exists a function $\eta:\R^M\to[0,\eta]$ such that samples $(x,y) \sim D$ are generated as follows:
    First, $x \sim D_x$ is drawn and then we output $(x,y)$ where $y = f(x)$ with probability $1-\eta(x)$ and $y = -f(x)$ with probability $\eta(x)$, i.e., $\eta(x)$ is the \emph{flipping probability}.
    
    In the problem of \emph{learning halfspaces in the Massart noise model}, we observe samples $(x,y)\sim D$ from an unknown distribution $D$ satisfying the $\eta$-Massart noise condition for some known bound $\eta \in [0,1/2]$, and the goal is to output a hypothesis $h \colon \R^M \rightarrow \set{-1,1}$ minimizing the \emph{misclassification error} $$\err{h} \coloneqq \Psymb_{(x,y) \sim D} \brac{h(x) \neq y}.$$
    Note that the marginal distribution $D_x$, the true hypothesis $f \colon \R^M \rightarrow \set{-1,1}$, and the flipping probability function $\eta:\R^M\to[0,\eta]$ are all unknown.
    
\end{definition}

The model was proposed in~\cite{MN06}.\footnote{Note that~\cite{RS94,S96} introduced an equivalent model called "malicious misclassification noise".}
Note that if $\eta(x) = \eta$ for all $x$, we obtain the Random Classification Noise model.
As previously mentioned, the information-theoretically optimal error in the RCN model is equal to $\eta$. However, in the more general case of $\eta$-Massart noise, the information-theoretically optimal error is equal to $$\OPT \coloneqq \Psymb_{(x,y) \sim D}[f(x) \neq y] = \E \eta(x)\,,$$ which can potentially be much smaller than $\eta$.
Information-theoretically, it was shown in~\cite{MN06} that for $\eta$ bounded away from $1/2$, a number $n = O\Paren{\frac{M \log(1/\e)}{(1-2\eta)^2\e}}$ of samples suffices to achieve misclassification error $\OPT + \e$ and that this is tight up to constants.
More generally, if the target halfspace is replaced by an unknown boolean function in a class of VC-dimension $d$, a number $n = O\Paren{\frac{d \log(1/\e)}{(1-2\eta)^2\e}}$ of samples suffices to achieve error $\OPT + \e$.
\footnote{We remark that previous works on algorithmic aspects of the Massart model stated this sample complexity as $O(d/\e^2)$. While this is correct, from~\cite{MN06} it follows that this is only necessary when $\eta \geq \frac{1}{2} \cdot (1 - \sqrt{d/n})$. For $\eta$ smaller than this the bound of $O\Paren{\frac{d \log(1/\e)}{(1-2\eta)^2\e}}$ holds.}

However, until recently, algorithmic results were only known when assuming that the marginal distribution of the examples $D_x$ belongs to some known class, e.g., is uniform  or log-concave~\cite{ABHU15,ABHZ16,ZLC17} or even more general in~\cite{DKTZ20}.
Under no assumption on the marginal distribution,~\cite{DGT19} was the first work that provided an efficient (improper) learning algorithm outputting a hypothesis $h$ (which is not a halfspace) such that $\err{h} \leq \eta + \e$.
They use time and sample complexities which are polynomial in $M$ and $\frac{1}{\e}$.
Building on this,~\cite{CKMY20} provided an efficient (proper) learning algorithm with the same error guarantees but whose output is itself a halfspace.
We remark that the sample complexity of both of the above works depends on the bit complexity of points in the support of $D_x$ although this is information-theoretically not necessary.
This assumption was recently removed in~\cite{DKT21}.
Further, the above works assume $\eta < 1/2$.
See~\cite{DKKTZ21} for a quasipolynomial algorithmic result without this assumption but under Gaussian marginal.

On the other hand, until very recently, no matching computational lower bounds were known and it remained an open question to determine whether it is possible to efficiently achieve error guarantees that are better than $\eta$, potentially going all the way to $\OPT$.
This question is especially intriguing since the above algorithmic results imply that non-trivial guarantees can be achieved in the Massart noise model, which is much more realistic than RCN.
The question then becomes if there are any computational limits at all in this model.
As we will see, such limits do indeed exist, at least when restricting to the class of Statistical Query algorithms.

\paragraph{Statistical Query Algorithms and Known Lower Bounds.}
Statistical Query (SQ) algorithms do not have access to actual samples from the (unknown) distribution $D$ but rather are allowed to query expectations of bounded functions over the underlying distribution.
These queries return the correct value up to some accuracy.
Since every such query can be simulated by samples from the distribution this is a restriction of Valiant's PAC model.
Note that a simple Chernoff bound shows that in order to simulate a query of accuracy $\tau$, a number of $O(1/\tau^2)$ samples is sufficient. 
Hence, SQ algorithms using $N$ queries of accuracy at most $\tau$ can be taken as a proxy for algorithms using $O(1/\tau^2)$ samples and running in time $\poly(N, 1/\tau)$.
The SQ model was originally introduced by~\cite{K98}. See~\cite{F16} for a survey.
Note, that it has also found applications outside of PAC learning, see e.g., \cite{KLNRS11, FGV21} for examples.

Intriguingly, \cite{K98} shows that any concept class that is PAC learnable in the realizable case using an SQ algorithm can also be learned in the PAC model under Random Classification Noise.
Further, almost all known learning algorithms are either SQ or SQ-implementable, except for those that are based on Gaussian elimination, e.g., learning parities with noise~\cite{K98, BKW03}.
One clear advantage of this framework is that it is possible to prove unconditional lower bounds.
This proceeds via the so-called \emph{SQ dimension} first introduced in~\cite{BFJKMR94} and later refined in~\cite{FGRVX17, F17}.
Although we will not see it explicitly, the lower bounds in this paper are also based on this parameter. See \cite{DK20} and the references therein for more detail.

For learning halfspaces under Massart noise, \cite{CKMY20} initiated the study of computational lower bounds.
The authors proved that when $\OPT$ is within a factor of 2 of $\eta$, achieving error $\OPT + \e$ requires superpolynomially many queries.
While this shows that obtaining optimal error is hard, it does not rule out the possibility of an efficient (SQ) algorithm achieving constant factor approximations.
More recently~\cite{DK20} proved that for $\tau = M^{-\omega(1)}$, achieving error better than $\Omega(\eta)$ requires queries of accuray better than $\tau$ or at least $1/\tau$ queries.
This holds even when $\eta$ is a constant but $\OPT$ goes to zero as the ambient dimension $M$ becomes large. This rules out any constant factor approximation algorithm, and also rules out efficient algorithms achieving error $O(\OPT^c)$ for any $c<1$.
Further, for $\eta$ close to $1/2$ the authors show that achieving error that is better than $\eta - o_\eta(1)$ for some term $o_\eta(1)$ that is constant in $M$, but depends on $\eta$ and goes to 0 as $\eta$ goes to 1/2, also requires super-polynomial time in the SQ framework.
For the special case of $\eta = 1/2$, \cite{DKKTZ21} shows that achieving error $\OPT + \e$ requires queries of accuracy better than $d^{-\Omega(\log (1/\e))}$ or at least $2^{d^{\Omega(1)}}$ queries even under Gaussian marginals.
However, as with~\cite{CKMY20}, this result only applies to exact learning.

As can be seen, the best previously known lower bounds are a constant-factor away from the best known algorithmic guarantees, but they do not match yet.
In the present work, we close this gap by showing that the algorithmic guarantees are actually tight, at least in the SQ framework.
More precisely, we will show that for arbitrary $\eta \in (0,1/2]$ any SQ algorithms that achieves error better than $\eta$ either requires a superpolynomial number of queries, or requires queries of superpolynomial accuracy.
Further, as for~\cite{DK20} the result holds even when $\OPT$ goes to zero as a function of the ambient dimension $M$ and $\eta$ is a constant arbitrarly close to 1/2.

\subsection{Results}

The following theorem is our main result (see~\cref{thm:main-full} for a more detailed version):
\begin{theorem}[Informal version]
    \label{thm:main}
    Let $M \in \R$ be sufficiently large and $\eta \in (0,1/2]$ be arbitrary.
    There exists no SQ algorithm that learns $M$-dimensional halfspaces in the $\eta$-Massart noise model to error better than $\eta$ using at most $\poly(M)$ queries of accuracy no better than $1/\poly(M)$.
    
    This holds even if the optimal halfspace achieves error $\OPT$ that vanishes as fast as $2^{-(\log M)^c}$ for some $c<1$, and even if we assume that all flipping probabilities are either $0$ or $\eta$.
\end{theorem}

Some remarks are in order:
\begin{itemize}
\item As we mentioned earlier, this lower bound matches the guarantees that are achievable in polynomial time~\cite{DGT19,CKMY20}.
Moreover, since these algorithms can be implemented in the SQ learning model, this completely characterizes the error guarantees that are efficiently achievable in the SQ framework for the class of halfspaces under Massart noise. Further, this also suggests that improving over this guarantee with efficient non-SQ algorithms might be hard.
\item For the special case $\eta = 1/2$, the result implies that handling $1/2$-Massart noise is as hard as the much more general agnostic model -- again for the class of halfspaces and in the SQ model.
Namely, it is hard to achieve error better than a random hypothesis.
Note that even though $\eta = 1/2$ means that there can be examples $x$ with completely random labels, the fact that $\OPT$ can be made go to zero implies that there would be a vanishing fraction of such examples.
We remark that Daniely gave a similar SQ lower bound for the agnostic model~\cite{D16}.
\item The fact that hardness still holds even if for all $x$ we have $\eta(x) \in \{0, \eta\}$ and even if $\OPT$ is very small implies that achieving error better than $\eta$ remains hard even if an overwhelming fraction of the samples have no noise in their labels. In light of the previous point this implies that even if the overwhelming majority of the points have no noise but the labels of just very few are random, outputting a hypothesis which does better than randomly classifying the points is SQ-hard.
\item The case when $\OPT = O\Paren{\log(M)/M}$ is computationally easy.
This follows since with high probability there is a subset of the observed samples in which no labels were flipped and which is sufficiently large to apply algorithms designed for the realizable case.
Hence, for values of $\OPT$ only slightly smaller than allowed by \cref{thm:main} achieving optimal misclassfication error is possible in polynomial time.
\end{itemize}

As a consequence of the above theorem, we immediately obtain strong hardness results for a more challenging noise model, namely the Tsybakov noise model~\cite{MT99,T04} defined as follows:
Let $A > 0$ and $\alpha \in [0,1)$.
Samples are generated as in the Massart model but the flipping probabilites $\eta(x)$ are not uniformly bounded by some constant but rather need to satisfy the following condition: $$\forall 0 < t \leq 1/2 \colon \Psymb [\eta(x) \geq 1/2 - t] \leq A \cdot t^{\alpha/(1-\alpha)} \,.$$
It is known that information-theoretically $O\Paren{\frac{A \cdot M}{\e^{2-\alpha}} \cdot \log(1/(A \e^\alpha))}$ samples suffice to learn halfspaces up to misclassification error $\OPT + \e$ in this model~\cite[Chapter 3]{H14}.
On the other hand, algorithmic results are only known when restricting the marginal distribution to belong to a fixed class of distributions (e.g., log-concave or even more general~\cite{DKKTZ21_tsybakov}).
On the other hand, we claim that our hardness result about Massart noise implies that it is SQ-hard to achieve error even slightly better than $1/2$ in the Tsybakov model.
Indeed, let $\zeta = 2^{-(\log M)^c}$ for some $0 < c < 1$.
Further, let $A$ be a constant, $\alpha \in (0,1)$ be bounded away from 1, and $$\eta = \frac{1}{2} - \Paren{\frac{\zeta}{A}}^{(1-\alpha)/\alpha} = \frac{1}{2} - \exp \Paren{-\Theta((\log M)^c)} \,.$$
Then the $\eta$-Massart condition together with the condition that $\eta(x) \in \Set{0,\eta}$ and $\OPT = \zeta$ implies the $(A,\alpha)$-Tsybakov condition.
To see this note that for $t \geq 1/2 - \eta$ we obtain that $$\Psymb [\eta(x) \geq 1/2 - t] \leq \Psymb [\eta(x) \geq 0] = \zeta = A \cdot \Paren{\frac{1}{2} - \eta}^{\alpha/(1-\alpha)} \leq A \cdot t^{\alpha/(1-\alpha)} \,,$$
and for $t < 1/2 - \eta$ we have $$\Psymb [\eta(x) \geq 1/2 - t] \leq \Psymb [\eta(x) > \eta] = 0 \leq A \cdot t^{\alpha/(1-\alpha)} \,.$$
Hence, by \cref{thm:main}, or \cref{thm:main-full}, achieving error better than $1/2 - \exp \Paren{-\Theta((\log M^c)}$ requires queries of accuray better than the inverse of any polymonial or at least superpolynomially many queries, even though $\OPT = 2^{-(\log M)^c}$.
Similarly, for $\alpha = 1 - 1/\log(M)^{c'}$ where $0 < c' < c$ it is hard to achieve error better than $1/2 - \exp \Paren{-\Theta(((\log M)^{c-c'})}$ in the sense above.
This stands in strong contrast to the fact that information-theoretically $\poly(M, 1/\e)$ samples and time suffice to achieve misclassification optimal error.
That is, even if the fraciton of flipping probabilites decreases very fast as we approach 1/2 learning in the model remains hard.

Lastly, we would like to mention that we closely follow the techniques developed in~\cite{DK20} (and previous works cited therein).
At the heart of their work one needs to design two distributions matching many moments of the standard Gaussian (and satisfying some additional properties).
The main difference in our work lies in how exactly we construct these distributions, which eventually leads to the tight result.

%% file: content/techniques.tex
\section{Techniques}
\label{sec:techniques}

In this section, we will outline the techniques used to prove~\cref{thm:main}.
On a high level, we will closely follow the approach of~\cite{DK20}.
First, note that for $M,m,d \in \N$ satisfying $M = \binom{m+d}{m}\le m^d$, any degree-$d$ polynomial over $x \in \R^m$ can be viewed as a linear function\footnote{We use an embedding $\R^m\to\R^M$ whose component functions are the monomials of degree $\leq d$.} over $\R^M$.
Hence, any lower bound against learning polynomial-threshold functions (PTFs) in $\R^m$ would yield a lower bound against learning halfspaces in $\R^M$.
Further, if we choose $m,d$ so that $m \approx \log(M)^{1+\alpha}$ for some constant $\alpha > 0$, then an exponential lower bound against learning PTFs in $\R^m$ would yield a superpolynomial lower bound against learning halfspaces in $\R^M$.

One key step of the SQ hardness result in~\cite{DK20} is to construct two specific distributions over $(x,y)\in \R^m\times\{-1,1\}$ and show that a mixture of these two distributions is SQ-hard to distinguish from a certain null distribution\footnote{The null distribution is the one where the example $x\in \R^m$ is standard Gaussian and the label $y\in \{-1,1\}$ is independent from $x$.}.
The authors then argue that any algorithm that learns $\eta$-Massart PTFs up to error better than $\Omega(\eta)$ can be used to distinguish these distributions from the null distribution.
We follow a similar proof strategy.
The main difference lies in how we construct the two hard distributions (in a simpler way), allowing us to obtain the optimal lower bound $\eta$.
In fact, we will show that two simple modifications of the standard Gaussian distribution will work.

Both distributions constructed in \cite{DK20} as well as the ones that we will construct have the following common structure:
Let $v \in \R^m$ be fixed but unknown, $p \in (0,1)$, and let $A,B$ be two one-dimensional distributions.
Define $D_+$ (respectively, $D_-$) as the distribution over $\R^m$ that is equal to $A$ (respectively, $B$) in the direction of $v$ and equal to a standard Gaussian in the orthogonal complement. Then, define the distribution $D$ over $\R^m \times \set{-1,1}$ as follows: With probability $p$ draw $x \sim D_+$ and return $(x,1)$, and with probability $1-p$ draw $x \sim D_-$ and return $(x,-1)$.
The goal is to output a hypothesis $h$ minimizing the misclassification error $\Psymb_{(x,y) \sim D} [h(x) \neq y]$. It is easy to see that one of the constant functions $1$ or $-1$ achieves error $\min \set{p,1-p}$. The question is whether it is possible to achieve error better than $\min \set{p,1-p}$.

Roughly speaking\footnote{This sweaps under the rock some details, see~\cref{sec:hardness_result} for all details.}, the authors of~\cite{DK20} show the following hardness result:
Suppose the first $k$ moments of $A$ and $B$ match those of $N(0,1)$ upto additive error at most $2^{-k}$ and their $\chi^2$-divergence with respect to $N(0,1)$ is not too large.
Then every SQ algorithm outputting a hypothesis achieving misclassification error slightly smaller than $\min \set{p,1-p}$ must either make queries of accuracy at least $2^{-k/2}$ or must make at least $2^{m-k}$ queries.
Hence, if we can choose $k$ to be a small constant multiple of $m$ we get an exponential lower bound as desired.
The authors then proceed to construct distributions satisfying the moment conditions with $\min \set{p,1-p}=\Omega(\eta)$ and such that $D$ corresponds to an $\eta$-Massart PTF.
In this paper, we construct distributions satisfying the moment conditions with $\min \set{p,1-p}=\eta$.
However, the $\chi^2$-divergence will be too large to apply the hardness result of~\cite{DK20} in a black-box way.
To remedy this, we show that its proof can be adapted to also work in this regime.
Further, by choosing the parameters slightly differently, the reduction still works.
In the following, we briefly describe our construction.
We will give a more detailed comparison with \cite{DK20} in \cref{sec:comparison}.

Let $0 < \eta \leq 1/2$ be the bound of the Massart model and fix $p = 1 - \eta$.
We will show that we can choose $A$ and $B$ satisfying the moment conditions above, in such a way that $D$ corresponds to an $\eta$-Massart PTF.
Note that this will directly imply~\cref{thm:main} via the previously outlined reduction.
We partition $\R$ into three regions $J_1, J_2$ and $\R \setminus (J_1 \cup J_2)$ such that the following conditions hold:
\begin{enumerate}
\item $A(x) = 0$ for $x \in J_2$ \label[condition]{cond_distr_1},
\item $B(x) = 0$ for $x \in J_1$ \label[condition]{cond_distr_2},
\item $A(x) \geq B(x)$ for all $x \in \R \setminus (J_1 \cup J_2)$ \label[condition]{cond_distr_3}.
\end{enumerate}
Suppose that $J_2$ can be written as the union of $d$ intervals and hence there is a degree-$2d$ polynomial $p:\mathbb{R}\to\mathbb{R}$ which is non-negative on $\R \setminus J_2$ and non-positive on $J_2$.
We claim, that $D$ is an $\eta$-Massart PTF for the polynomial $p_v:\mathbb{R}^m\to\mathbb{R}$ defined as $$p_v(x) = p(\iprod{v,x}) \,.$$
Let $D_x(x) \coloneqq \sum_y D(x,y)$ be the marginal distribution of $D$ on $x$.
Then this means that for all $x \in \R^m$, such that $D_x(x) > 0$ it needs to hold that $$\eta(x) \coloneqq \Psymb_{(x,y) \sim D} \Brac{y \neq \sign (p_v(x)) \suchthat x} \leq \eta \,.$$
Indeed, consider $x$ such that $\iprod{x,v} \in J_1$.
Since $p_v(x) \geq 0$ and $B(\iprod{x,v}) = 0$ it follows that $\eta(x) = 0$.
On a high level, this is because none of the samples with label $-1$ lie in this region.
Similarly, the same holds for $x$ such that $\iprod{x,v} \in J_2$.
Now consider $x \in \R^m$ such that $D_x(x) > 0$ and $\iprod{x,v} \in \R \setminus (J_1 \cup J_2)$.
Since $\sign(p_v(x)) = 1$ and $A(x) \geq B(x)$ it follows
\begin{equation}
    \label{eq:massart}
    \begin{split}
        \Psymb_{(x,y) \sim D} \Brac{y \neq \sign(p(x)) \suchthat x} &= \frac{\Psymb_{(x,y) \sim D}[y \neq \sign(p(x)), x]}{D_x(x)} = \frac{(1-p) \cdot B(x)}{p \cdot A(x) + (1-p) \cdot B(x)} \\
        &\leq 1-p = \eta \,.
    \end{split}
\end{equation}
Note that in our cosntruction it will actually hold that $A(x) = B(x)$ for all $x \in \R \setminus (J_1 \cup J_2)$.
Hence, it even holds that $\eta(x) \in \Set{0,\eta}$ for all $x$.

Our work crucially departs from~\cite{DK20} in our choice of $A$ and $B$ to satisfy \crefrange{cond_distr_1}{cond_distr_3} and the moment-matching condition.
In fact, giving a very clean construction will turn out to be essential for achieving the tightest possible lower bound.
On a high level, $A$ will be equal to an appropriate multiple of the standard Gaussian distribution on periodically spaced intervals of small size and 0 otherwise.
$B$ will be the equal to $A$ for $x$ of large magnitude.
For smaller $x$ we will slightly displace the intervals.

Concretely, let $0 < \delta, \e<1$ be such that $\e < \delta/8$ and consider the infinte union of intervals $$J = \bigcup_{n \in \Z} \,[n\delta - \e, n\delta + \e] \,.$$
Denote by $G$ the pdf of a standard Gaussian distribution.
We define (the unnormalized measures) 
\begin{align*}
    A(x) = \begin{cases} \frac{\delta}{2\e} \cdot G(x)\, , &\quad\text{if } x \in J \,, \\ 0\,, &\quad \text{otherwise.} \end{cases} \, && B(x) = \begin{cases} A(x)\,, &\quad \text{if } \abs{x} > d \delta + 5\e\,, \\ A(x+4\e)\,, &\quad \text{otherwise.}  \end{cases}
\end{align*}
Clearly, the total probability mass of the two is the same.
It can be shown that it is $1\pm\exp(-\Omega(1/\delta)^2)$, so for the sake of this exposition assume that it is exactly one and that $A$ and $B$ are in fact probability distributions (see \cref{sec:hard_distributions} for all details).
Further, consider
\begin{align*}
    J_1 = \bigcup_{n = -d}^{d} [n\delta - \e, n\delta + \e]\,, &&  J_2 = \bigcup_{n = -d}^{d} [n\delta - 5\e, n\delta - 3\e]\,.
\end{align*}
It is not hard to verify that $A,B$ together with $J_1, J_2$ satisfy \crefrange{cond_distr_1}{cond_distr_3}.
Hence, our final distribution $D$ will satisfy the Massart condition.
Since $\eta(x) \neq 0$ only if $A(x) = B(x) > 0$ which only is the case when $\abs{x} \gtrsim d\delta$ it follows that $\OPT$ is very small as well.\footnote{Note, that here $J_2$ is the union of $2d+1$ intervals. It is straightforward to adapt the previous discussion to this case.}

The fact that the moments of $A$ match those of a standard Gaussian will follow from the fact that it is obtained by only slightly modifying it.
This part is similar to \cite{DK20}.
Note that $B$ is equal to $A$ for $x$ of magnitude larger than roughly $d\delta$ and for smaller $x$ is obtained by displacing $A$ by $\e$.
Hence, it will follow that its first $k$ moments match those of $A$ (and hence also those of a standard Gaussian) up to error $\e (d\delta)^k$.
In~\cref{sec:hardness_result}, we will show that we can choose the parameters such that for $k$ slightly smaller than $m$ we can make the first $k$ moments of $A$ and $B$ match those of a standard Gaussian up to error at most roughly $\exp(-\Omega(m))$ which will be sufficient.

\subsection{Comparison with \cite{DK20}}
\label{sec:comparison}
The key property that allowed us to achieve the sharp lower bound of $\eta$ was that $A(x) \geq B(x)$ on $\R \setminus (J_1 \cup J_2)$. Indeed, if we only had $A(x) \geq c \cdot B(x)$ for some constant $0 < c < 1$, the resulting distribution $D$ would no longer be $\eta$-Massart (cf.~\cref{eq:massart}), and the only way to still make it so is to increase $p$ which in turn degrades the resulting lower bound. More precisely, if we only have $A(x) \geq c \cdot B(x)$, then the upper bound in~\cref{eq:massart} will now be $\frac{1-p}{c \cdot p + 1-p}$ instead of $1-p$.
Basic manipulations show that this is less than or equal to $\eta$ if and only if $p \geq \frac{1}{1-\eta(1-c)} \cdot (1-\eta) > 1 - \eta$, which means that the lower bound that we get from the distinguishing problem is at best $\min \set{p, 1-p} = \Omega(\eta)$.

While our construction can avoid this issue because we can ensure that $A(x) \geq B(x)$ for $x\notin J_1\cup J_2$ (in fact, we will have $A(x) = B(x)$), it is unclear if the same can be achieved using the construction of~\cite{DK20}, or a slight modification of it. In their work, the supports of $A$ and $B$ also consist of unions of intervals, but they increase in size as we move away from the origin. The intervals of $A$ are disjoint from those of $B$ for $x$ of small magnitude, but they start to overlap when $|x|$ becomes large.
On each interval the distribution is also a constant multiple of $G(x)$, however, their specific choice makes exact computations difficult and the authors only show that $A(x) \geq \Omega(B(x))$ where the constants in the $\Omega$-notation can be smaller than 1.~\footnote{We remark that the authors do not work with distributions directly but with unnormalized measures. Normalizing them does not change the construction but makes the comparison easier.}
We note, however, that the moment bounds the authors use for their distribution are very similar to the one we use for our distribution $A$.

On a more technical level, we cannot directly apply the hardness result \cite[Proposition 3.8]{DK20} the authors used.
Suppose the first $k$ moments of $A$ and $B$ match those of a standard Gaussian up to additve error $\nu$ and the $\chi^2$-divergence of $A$ and $B$ with respect to the standard Gaussian is at most $\alpha/2$.
Further, let 
\begin{align*}
    \tau = \nu^2 + 2^{-k}\alpha \,, && N = 2^{\Omega(m)} \tau/\alpha \,.
\end{align*}
Then this result says that for every SQ algorithm achieving misclassification error better than $\min \Set{p,1-p} - 4\sqrt{\tau}$ must either make queries of accuracy better than $2 \sqrt{\tau}$ or must make at least $N$ queries.
Since in our construction we need to choose $\e$ sufficiently small to match many moments --- which in turn will increase the $\chi^2$-divergence --- we will have $\alpha \gg 2^k$ which is too large for the above.
On the flip side, the proof of \cite[Proposition 3.8]{DK20} can readily be adapted (in fact, this is already implicit in the proof) so that the same conclusion also holds for
\begin{align*}
    \tau = \nu^2 + c^k \alpha \,, && N = 2^{c^2 \cdot \Omega(m)} \tau/\alpha \,,
\end{align*}
for some arbitrarily small $c$ where the constant in $\Omega(m)$ is independent of $c$.
It will turn out that we can choose $c$ sufficiently small and in turn $m$ slightly larger so that the above yields the desired bounds.
See \cref{sec:hardness_result} and \cref{sec:generic_sq_lower_bound} for an in-depth discussion.

%% file: content/preliminaries.tex
\section{Preliminaries}
\label{sec:preliminaries}

For two functions $f,g \colon \R \rightarrow \R$, we will write $f \ll g$ if $\displaystyle\lim_{x \rightarrow \infty} \frac{f(x)}{g(x)} = 0$.
Similarly, we will write $f \gg g$ if $\displaystyle\lim_{x \rightarrow \infty} \frac{f(x)}{g(x)} = \infty$.

All logarithms will be to the base $e$.

We will use $N(0,1)$ to denote the one-dimensional standard Gaussian distribution.
We will denote its pdf by $G$ and with a slight abuse of notation we will also refer to a standard Gaussian random variable by $G$.

For two probability distribution $A$ and $B$ we denote their $\chi^2$-divergence by $$\chi^2 (A, B) = \int_{-\infty}^\infty \frac{A(x)^2}{B(x)} \, dx - 1\,.$$
For an unnormalized positive measure $A$ we denote its total measure by $\normo{A}$.

%% file: content/hardness_result.tex
\section{Hardness Result}
\label{sec:hardness_result}

In this section, we will prove the full version of~\cref{thm:main}.
Concretely, we will show that
\begin{theorem}
    \label{thm:main-full}
    Let $0<\zeta \leq \eta \leq \frac{1}{2}$ and $M \in \N$ be such that $\displaystyle l \coloneqq \frac{\log M}{(\log \log M)^3\log (1/\zeta)}$ is at least a sufficiently large constant.
    There exists a parameter $\tau \coloneqq M^{-\Theta(l)}$ for which there is no SQ algorithm that learns the class of halfspaces on $\R^M$ with $\eta$-Massart noise using at most $1/\tau$ queries of accuracy $\tau$ and which achieves misclassification error that is better than $\eta - \tau$.
    This holds even if the optimal halfspace has misclassification error that is as small as $\zeta$ and all flipping probabilites are either 0 or $\eta$.
\end{theorem}
Note that $\zeta = 2^{-\log(M)^c}$ with $0 < c < 1$ satisfies the assumption of the theorem and we recover \cref{thm:main}.
As previously mentioned, the setting is the same as in~\cite{DK20} except that we achieve a lower bound of $\eta$.

We will prove~\cref{thm:main-full} by reducing it to the following classification problem, which was introduced in~\cite{DK20}, and then applying a lower bound that was proved in the same reference.

\begin{definition}[Hidden Direction Classification Problem]
    \label{def:binary_classification_problem}
    Let $A,B$ be two probability distributions over $\R$, let $p \in (0,1)$, and $v$ be a unit vector in $\R^m$.
    Let $D_+$ (respectively $D_-$) be the distribution that is equal to $A$ (respectively $B$) in the direction of $v$ and equal to a standard Gaussian in its orthogonal complement.
    Consider the distribution $D_v^{A,B,p}$ on $\R^m \times \Set{-1,1}$ defined as follows: With probability $p$ draw $x \sim D_+$ and output $(x,1)$, and with probability $1-p$ draw $x \sim D_-$ and return $(x,-1)$.
    The \emph{Hidden Direction Classification Problem} is the following: Given sample access to $D_v^{A,B,p}$ for a fixed but unknown $v$, output a hypothesis $h \colon \R^m \rightarrow \Set{-1,1}$ (approximately) minimizing $\Psymb_{(x,y) \sim D_v^{A,B,p}}[h(x) \neq y]$.
\end{definition}

Achieving misclassification error $\min \set{p,1-p}$ can trivially be achieved by one of the constant functions 1 or $-1$.
The following lemma shows that in the SQ framework, one cannot do better if the distributions $A$ and $B$ (approximately) match many moments of the standard Gaussian distribution.
Its proof is analogous to the one of Proposition 3.8 in~\cite{DK20}. We will give a more detailed discussion in \cref{sec:generic_sq_lower_bound}.
\begin{lemma} [Adaptation of Proposition 3.8 in~\cite{DK20}]
    \label{lem:generic_SQ_lower_bound}
    Let $k \in \N$ and $\nu,\rho, c > 0$.
    Let $A,B$ be probability distributions on $\R$ such that their first $k$ moments agree with the first $k$ moments of $N(0,1)$ up to error at most $\nu$ and such that $\chi^2(A, N(0,1))$ and $\chi^2(B, N(0,1))$ are finite.
    Denote $\alpha \coloneq \chi^2(A, N(0,1)) + \chi^2(B, N(0,1))$ and assume that $\nu^2 + \alpha \cdot c^k \leq \rho$.
    Then, any SQ algorithm which, given access to $D_v^{A,B,p}$ for a fixed but unknown $v \in \R^m$, outputs a hypothesis $h \colon \R^m \rightarrow \Set{-1,1}$ such that $$\Psymb_{(x,y) \sim D_v^{A,B,p}}[h(x) \neq y] < \min\set{p,1-p} - 4\sqrt{\rho}\,,$$ must either make queries of accuracy better than $2\sqrt{\rho}$ or make at least $N = 2^{c^2 \cdot \Omega(m)} \cdot (\rho/\alpha)$ queries.
\end{lemma}

The goal is now to find distributions $A,B$ satisfying the conditions of~\cref{lem:generic_SQ_lower_bound} and such that the distribution $D_v^{A,B,p}$ corresponds to the Massart noise model.
To this end, consider distributions $A,B$ and unions of intervals $J_1,J_2$ given by the following theorem which we will prove in \cref{sec:hard_distributions}.

\begin{proposition}
    \label{prop:hard_distributions}
    Let $0 < \zeta < 1/2$ and let $d\geq 2$ be an integer. Define $\delta = 4\sqrt{\log(1/\zeta)}/d$ and let $\epsilon < \delta/8$.
If $\delta<1$, there exist probability distributions $A, B$ on $\R$ and two unions $J_1, J_2$ of $2d+1$ intervals such that
    \begin{enumerate}
        \item $J_1 \cap J_2 = \emptyset$ and $J_1 \cup J_2 \subseteq [- d\delta-5\e, d\delta+5\e] $ \label{cond_1},
        \item (a) $A = 0$ on $J_2$, $B = 0$ on $J_1$, and (b) for all $x \not\in J_1 \cup J_2$ we have $A(x) = B(x)$ \label{cond_2},
        \item for all $k \in \N$ the first $k$ moments of $A$ and $B$ match those of a standard Gaussian within additive error $O(k!) \cdot \mathrm{exp}(-\Omega(1/\delta^2)) + 4\epsilon \Paren{12 \sqrt{\log(1/\zeta))}}^k$ \label{cond_3},
        \item at most a $\zeta$-fraction of the measure $A$ (respectively $B$) lies outside $J_1$ (respectively $J_2$) \label{cond_4},
        \item $\chi^2\Paren{A, N(0,1)}= O \Paren{\frac{\delta}{\e}}^2$ and $\chi^2\Paren{B,N(0,1)} = O \Paren{\frac{\delta}{\e}}^2$ \label{cond_5}.
    \end{enumerate}
\end{proposition}

Although $D_v^{A,B,p}$ will not correspond to a Massart distribution when considering only halfspaces, it will turn out to work when considering polynomial threshold functions, i.e., $y = \sign(p(x))$ for some polynomial $p$ in $x$.
Further, we will be able to choose the parameters such that \cref{lem:generic_SQ_lower_bound} will correspond to a super-polynomial lower bound in terms of $M$.

Unless explicitly indicated by a subscript, in what follows the $O(\cdot), \Theta(\cdot), \Omega(\cdot)$-notation will only contain universal constants indpendent of the ones we define throughout the section.
Fix a unit vector $v \in  \R^m$ and let $0 < \zeta < \eta$ be such that $$\frac{\log M}{(\log \log M)^3} \geq C_\zeta \log (1/\zeta)$$ for a sufficiently large constant $C_\zeta$.
Further, let $$\tau = M^{-\frac{\log M}{C_\tau(\log \log M)^3 \log (1/\zeta)}}$$
for a sufficiently large constant $C_\tau$, so that $$\log(1/\tau) = \frac{(\log M)^2}{C_\tau(\log \log M)^3 \log (1/\zeta)}\,.$$
We would like to find $m$ and $d$ such that we can represent degree-$8d$ polynomials over $\R^m$ as halfspaces over $\R^M$.
It is sufficient to have $$\binom{8d + m}{8d} \leq m^{8d} \leq M.$$
To this end, for $C_m$ and $C_d$ sufficiently large constants, consider $$m = \left\lceil C_m \log(1/\tau) \log(1/\zeta)^4\right\rceil$$ and $$d = \left\lceil C_d\sqrt{\log (1/\zeta) \log(1/\tau) \log \log(1/\tau)} \right\rceil\,.$$
Notice that since $$\log(1/\tau) \geq \frac{C_\zeta^2 \cdot \Paren{\log(1/\zeta)}^2 \cdot \Paren{\log \log M}^3 }{C_\tau \log (1/\zeta)} \gg \log(1/\zeta)$$
it follows that $$\log m = \log \log(1/\tau) + 4 \log \log(1/\zeta) + \Theta_{C_m}(1) = \Theta_{C_m}(\log \log (1/\tau))\,.$$
Hence,
\begin{align*}
    m^{8d} &= \exp(8d \cdot \log m) = \exp\Paren{\Theta_{C_m,C_d}\Paren{\sqrt{\log(1/\zeta) \log(1/\tau) \big(\log \log (1/\tau)\big)^3}}} \\
    &= \exp\Paren{\frac{1}{\sqrt{C_\tau}}\cdot\log (M) \cdot \Theta_{C_m,C_d}\Paren{\frac{\log \log (1/\tau)}{ \log \log M}}^{3/2}} \leq M\,,
\end{align*}
where the last inequality follows since $\log \log (1/\tau) \leq 2\log \log (M)$ and by choosing $C_\tau$ to be large enough with respect to $C_m$ and $C_d$.
Let 
$$\delta = \frac{4\sqrt{\log(1/\zeta)}}{d} = \Theta\Paren{\frac{1}{C_d\sqrt{\log(1/\tau) \log \log(1/\tau)}}} \,.$$
Further, let
\begin{align*}
    k = \frac{4\log(1/\tau)}{\log \log(1/\zeta)} \,, && \e = \tau \cdot \Paren{12\sqrt{\log(1/\zeta)}}^{-k} \,.
\end{align*} 
and consider the probability distributions $A, B$ defined by \cref{prop:hard_distributions} for our settings of $\delta, \zeta,$ and $\e$.
Also, let $J_1, J_2$ be the corresponding unions of intervals.
Let $$D^{(m)} \coloneqq D_v^{A,B,p}\text{ with }p = 1-\eta\,,$$
so that $\min\set{p,1-p} = \eta$.
As we will shortly see, $D^{(m)}$ is an $\eta$-Massart polynomial-threshold function. In order to obtain an $\eta$-Massart halfspace, we will embed $D^{(m)} $ into the higher dimensional space $\R^M$.

Let $$M'\coloneq\binom{m+8d}{8d}\leq m^{8d}\leq M\,,$$
and define 
\begin{align*}
V_{8d}: \R^m&\rightarrow\;\R^{M'}\\
x\;\; &\mapsto (x^\alpha)_{\abs{\alpha} \;\leq 8d}\,,
\end{align*}
where $\alpha=(\alpha_1,\ldots,\alpha_m)\in\N^m$ is a multi-index and $|\alpha|= \sum_{i\in[m]}\alpha_i$. Furthermore, let
\begin{align*}
E_{M'\to M}: \R^{M'}&\rightarrow\;\R^{M}\\
x\;\; &\mapsto (x,0)\,,
\end{align*}
be the linear embedding of $\R^{M'}$ into $\R^{M}$ that is obtained by appending by zeros. We will embed  $D^{(m)}$ into $\R^{M}$ using the embedding $E:\R^m\to\R^M$ defined as
$$E = E_{M'\to M}\circ V_{8d}.$$

The hard distribution $D$ is as follows: Draw $(x,y) \sim D^{(m)}$ and return $(E(x), y)$.
The next lemma shows that this distribution satisfies the $\eta$-Massart property with respect to the class of halfspaces.
\begin{lemma}
    \label{lem:hard_distribution_is_massart}
    The probability distribution $D$ is an $\eta$-Massart halfspace with $\OPT \leq \zeta$.
\end{lemma}
\begin{proof}
    Let $v \in \R^m$ and consider the function $g_v \colon \R^m \rightarrow \set{-1,1}$ such that 
    $$g_v(x) = \begin{cases} -1\,,\quad&\text{if }\iprod{v,x} \in J_2\,,\\+1\,,\quad&\text{otherwise.}\end{cases}$$
    Since $J_2$ is a union of $2d+1 \leq 4d$ intervals, $g_v$ can be written as $\sign(p_v(x))$, where $p_v(x)=p(\iprod{v,x})$ for some degree-$8d$ polynomial $p$.
    Now since $M' = \binom{m+8d}{m}$, there is a linear function $f \colon \R^{M'} \rightarrow \set{-1,1}$ such that for all $x \in \R^m$ it holds that $g_v(x) = \sign(f(V_{8d}(x)))$. This in turn implies that there is a linear function $h \colon \R^{M} \rightarrow \set{-1,1}$ such that for all $x \in \R^m$ we have $g_v(x) = \sign(h(E(x)))$.

    Note that $D(x',y) \neq 0$ only if $x' = E(x)$ for some $x \in \R^m$.
    Furthermore,
\begin{itemize}
\item For $x \in \R^m$ satisfying $\iprod{x,v} \in (J_1 \cup J_2)$, we have $y = \sign(h(E(x)))$ with probability 1.
\item For $x \in \R^m$ satisfying $\iprod{x,v} \not\in (J_1 \cup J_2)$ and $D(E(x),y) \neq 0$, we have $\sign(h(E(x))) = 1$, and
$$y = \begin{cases} 1=\sign(h(E(x))) &\quad\text{with probability }p = 1-\eta,\\-1=-\sign(h(E(x)))&\quad\text{with probability }1-p =\eta.\end{cases}$$
\end{itemize}    
     
Hence, $D$ corresponds to an $\eta$-Massart distribution corresponding to the halfspace $f$.
Furthermore, the flipping probability function $\eta:\R^M\to[0,\eta]$ satisfies
$$\eta(x')=\begin{cases}\eta\,, \quad&\text{if }\;\exists x\in R^m,\, x'=E(x)\text{ and }\iprod{x,v} \notin (J_1 \cup J_2)\,.\\0\,,\quad&\text{otherwise.}\end{cases}$$
Therefore, $\eta(x') \in \Set{0,\eta}$ for all $x' \in \R^M$ and $$\OPT = \E_{(x',y) \sim D} \eta(x') = \eta \cdot \Psymb_{x\sim D^{(m)}}[\iprod{v,x} \not\in (J_1 \cup J_2)] \leq \eta \cdot \zeta \leq \zeta.$$
The second last inequality follows from \cref{cond_2} and \cref{cond_4} of \cref{prop:hard_distributions}. Indeed, we have $$\Psymb_{x\sim D^{(m)}}[\iprod{v,x} \not\in (J_1 \cup J_2)] = p \cdot \Psymb_{\iprod{v,x} \sim A}[\iprod{v,x} \not\in (J_1 \cup J_2)] + (1-p) \cdot \Psymb_{\iprod{v,x} \sim B}[\iprod{v,x} \not\in (J_1 \cup J_2)] \leq \zeta \,.$$
\end{proof}

Second, any hypothesis for predicting $y$ from $x$ can be turned into one predicting $y$ from $E(x)$ and vice-versa.
Hence, it is enough to show that the former is SQ-hard.
Consider the setting of~\cref{lem:generic_SQ_lower_bound} with $A$ and $B$ given by \cref{prop:hard_distributions}.
First, by \cref{cond_5} we know that $\chi^2(A,N(0,1))=O\Paren{\frac{\delta}{\e}}^2$ and $\chi^2(B,N(0,1)) = O\Paren{\frac{\delta}{\e}}^2$.
Hence, $$\alpha = O\Paren{\frac{\delta}{\e}}^2\,.$$

Further, let $$\gamma = O(k!) \cdot \exp \Paren{-\Omega(1/\delta^2)} \,.$$
By~\cref{prop:hard_distributions} we know that the first $k$ moments of $A$ and $B$ match those of a standard Gaussian up to additive error $$\gamma + 4\e \Paren{12 \sqrt{\log(1/\zeta)}}^k = \gamma + 4\tau \,,$$
where the equality follows from the fact that $\e = \tau \cdot \Paren{12\sqrt{\log(1/\zeta)}}^{-k} $.

We claim that by choosing $C_d$ large enough we get $\gamma \ll \tau$.
Indeed, since $\displaystyle k = \frac{4\log(1/\tau)}{\log \log(1/\zeta)} $ and $\displaystyle\delta =  \Theta\Paren{\frac{1}{C_d\sqrt{\log(1/\tau) \log \log(1/\tau)}}}$, we have
\begin{align*}
    \log(O(k!)) &= O(k \log k) = O\Paren{\log(1/\tau) \frac{\log \log(1/\tau)}{\log \log (1/\zeta)}} \leq  O\Paren{\frac{1}{C_d^2\delta^2}}\,.
\end{align*}
Hence, by choosing $C_d$ large enough, we get $$\gamma = \exp \Paren{-\Omega(1/\delta^2)} = \exp \Paren{-\Omega(C_d^2\log(1/\tau) \log \log(1/\tau))} \ll \tau \,.$$
It follows that both $A$ and $B$ match the moments of a standard Gaussian up to additive error at most $\nu=5\tau$. This, in addition to the fact that $\e = \tau \cdot \Paren{12\sqrt{\log(1/\zeta)}}^{-k} $, imply that the parameter $\nu^2+\alpha\cdot c^k$ in \cref{lem:generic_SQ_lower_bound} is equal to 
$$\nu^2+\alpha c^k = 25\tau^2 + O\Paren{\frac{\delta}{\e}}^2 \cdot c^k 
\leq 25\tau^2 + \frac{O(1)}{\e^2} \cdot c^k 
\leq 25\tau^2 + \frac{O(1)}{\tau^2} \cdot \Paren{144c \cdot \log(1/\zeta)}^k.$$

By choosing $\displaystyle c = \frac{1}{144\log(1/\zeta)^2}$ and recalling that $\displaystyle k = \frac{4\log(1/\tau)}{\log \log(1/\zeta)} $,  we get 
$$\Paren{144 c \cdot \log(1/\zeta)}^k = \log (1/\zeta)^{-k} = \exp \Paren{-\log \log (1/\zeta) \cdot \frac{4\log(1/\tau)}{\log \log (1/\zeta)}} = \tau^4 \,,$$
which implies that $$\nu^2+\alpha\cdot c^k \leq O(\tau^2)\leq\tau$$ for sufficiently large $M$ (and hence sufficiently small $\tau$). Therefore, we can choose the parameter $\rho$ in \cref{lem:generic_SQ_lower_bound} to be equal to $\tau$.

Next, we claim that the parameter 
$$N = \frac{2^{c^2 \cdot \Omega(m)}\cdot \rho}{\alpha} \geq \Omega\Paren{\frac{\e}{\delta}}^2 \cdot \exp(c^2 \cdot \Omega(m))\cdot \tau \geq \Omega(\e^2)\cdot \exp\big(c^2 \cdot \Omega(m)-\log(1/\tau)\big)$$
of \cref{lem:generic_SQ_lower_bound} is at least $1/\tau^{\Theta(1)}$.
In fact, recalling that $\displaystyle c = \frac{1}{144\log(1/\zeta)^2}$, $\e = \tau \cdot \Paren{12\sqrt{\log(1/\zeta)}}^{-k} $, $\displaystyle k = \frac{4\log(1/\tau)}{\log \log(1/\zeta)} $ and $m = \left\lceil C_m \log(1/\tau) \log(1/\zeta)^4\right\rceil$, we obtain
\begin{align*}
    \Omega(\e^2) \exp\big(c^2 \cdot \Omega(m)-\log( 1/\tau)\big) &= \exp\Paren{c^2 \cdot \Omega(m) - 3 \log(1/\tau) - 2 k \log\Paren{12 \sqrt{\log(1/\zeta)}} -O(1)} \\
    &\geq \exp \Paren{C_m \cdot \Omega(1) \cdot \log(1/\tau) - 3 \log(1/\tau) - \Theta(k \log \log (1/\zeta))}\\
    & = \exp \Paren{(C_m \cdot \Omega(1) - \Theta(1)) \cdot \log(1/\tau)}\,.
\end{align*}
By choosing the constant $C_m$ in the definition of $m$ large enough, we conclude that $N \geq 1/\tau^{\Theta(1)}$.

Hence, by~\cref{lem:generic_SQ_lower_bound} any SQ algorithm that outputs a hypothesis $h$ such that $$\Psymb_{(x,y)\sim D}[h(x) \neq y] \leq \eta - 4\sqrt{\tau}$$ must either make queries of accuracy better than $2\sqrt{ \rho} = \tau^{\Theta(1)}$ or make at least $N = 1/\tau^{\Theta(1)}$ queries.
Since $$\tau^{\Theta(1)} = M^{-\Theta\Paren{\log M/\left[\log(1/\zeta) \cdot (\log \log M)^3\right]}}\,,$$ \cref{thm:main-full} follows.

%% file: content/hard_distributions.tex
\subsection{Hard Distributions}
\label{sec:hard_distributions}

In this section, we will construct the one-dimensional moment-matching distributions.
Concretely, we will show the following proposition:
\begin{proposition}[Restatement of \cref{prop:hard_distributions}]
    Let $0 < \zeta < 1/2$ and let $d\geq 2$ be an integer. Define $\delta = 4\sqrt{\log(1/\zeta)}/d$ and let $\epsilon < \delta/8$.
If $\delta<1$, there exist probability distributions $A, B$ on $\R$ and two unions $J_1, J_2$ of $2d+1$ intervals such that
    \begin{enumerate}
        \item $J_1 \cap J_2 = \emptyset$ and $J_1 \cup J_2 \subseteq [- d\delta-5\e, d\delta+5\e] $ \label{cond_1_dup},
        \item (a) $A = 0$ on $J_2$, $B = 0$ on $J_1$, and (b) for all $x \not\in J_1 \cup J_2$ we have $A(x) = B(x)$ \label{cond_2_dup},
        \item for all $k \in \N$ the first $k$ moments of $A$ and $B$ match those of a standard Gaussian within additive error $O(k!) \cdot \mathrm{exp}(-\Omega(1/\delta^2)) + 4\epsilon \Paren{12 \sqrt{\log(1/\zeta))}}^k$ \label{cond_3_dup},
        \item at most a $\zeta$-fraction of the measure $A$ (respectively $B$) lies outside $J_1$ (respectively $J_2$) \label{cond_4_dup},
        \item $\chi^2\Paren{A, N(0,1)}= O \Paren{\frac{\delta}{\e}}^2$ and $\chi^2\Paren{B,N(0,1)} = O \Paren{\frac{\delta}{\e}}^2$ \label{cond_5_dup}.
    \end{enumerate}
\end{proposition}

Our construction will be based on the measure $G_{\delta,\epsilon}$ of density
\begin{align*}
G_{\delta,\epsilon}(x) = \sum_{n \in \Z} G(x)\cdot \Paren{\frac{\delta}{2\epsilon}} \cdot \ind{x \in [n\delta - \e, n\delta + \e]}\,,
\end{align*}
where $G$ is the standard Gaussian measure (and by abuse of notation, its density).
 Let $$G_{\delta,\e}^{(n)}(x) = \frac{G_{\delta,\epsilon}(x)}{\normo{G_{\delta,\epsilon}}}$$ be the probability distribution obtained by normalizing the measure $G_{\delta,\epsilon}$.
We define
\begin{align*}
    A(x) = G_{\delta,\e}^{(n)}(x)\,, && B(x) = \begin{cases} A(x)\,, &\quad \text{if } \abs{x} > d \delta + 5\e\,, \\ A(x+4\e)\,, &\quad\text{otherwise.} \end{cases}
\end{align*}
and
\begin{align*}
    J_1 = \bigcup_{-d\leq n\leq d} [n\delta - \e, n\delta + \e] \,,&& J_2 = \bigcup_{-d\leq n\leq d} [n\delta - 5\e, n\delta - 3\e] \,.
\end{align*}
Since $\epsilon<\delta/8$, \cref{cond_1_dup}  and \cref{cond_2_dup} of \cref{prop:hard_distributions} clearly hold.

In order to show \cref{cond_4_dup}, we bound the measure of $G_{\delta, \e}^{(n)}$ outside $J_1\cup J_2$ by $\zeta$.
Indeed, it then follows that 
\begin{align*}
    \int_{x \not\in J_1} A(x) \, dx = \int_{x \not\in J_2} B(x) \, dx =     \int_{x \not\in J_1\cup J_2} G_{\delta, \e}^{(n)}(x) \, dx \leq \zeta\,.
\end{align*}
In order to do so, we will upper bound the measure of $G_{\delta, \e}$ outside $J_1\cup J_2$ and will lower bound the total measure $\normo{G_{\delta,\epsilon}}$. We have:
\begin{align*}
    \int_{x \not\in J_1\cup J_2} G_{\delta, \e}(x) \, dx &\leq 2\int_{d\delta - 5\e}^\infty G_{\delta, \e} (x) \, dx = 2 \Paren{\frac{\delta}{2\e}} \sum_{n > d}\int_{n\delta - \e}^{n \delta +\e} G(x) \, dx \leq 2 \sum_{n > d} \int_{(n-1) \delta + \e}^{n \delta + \e} G(x) \, dx \\
    &\leq 2 \mathbb{P}\big[N(0,1)\geq d\delta \big] \leq   2 \exp\big(-(d\delta)^2/2\big)\leq 2\exp\Paren{-\Paren{4\sqrt{\log(1/\zeta)}}^2/2}\\
    &\leq 2\exp\Paren{-8\log(1/\zeta)} = 2\zeta^{8}\,,
\end{align*}
where we used the fact that $G(x)$ is decreasing for $x \geq 0$ and that $d = 4\sqrt{\log(1/\zeta)}/\delta\geq 2$. As for $\normo{G_{\delta,\epsilon}}$, we have
\begin{equation}
\label{eq:boundGdeTotalMeasure}
\begin{aligned}
\normo{G_{\delta,\epsilon}} &= \int_{\R} G_{\delta, \e}(x) \, dx \geq 2\int_{\delta - \e}^\infty G_{\delta, \e} (x) \, dx = 2 \Paren{\frac{\delta}{2\e}} \sum_{n \geq 1}\int_{n\delta - \e}^{n \delta +\e} G(x) \, dx \\
   &\geq  2 \sum_{n \geq 1}\int_{n\delta - \e}^{(n+1) \delta -\e} G(x) \, dx =  2 \mathbb{P}\big[N(0,1)\geq\delta-\e \big] \geq 2 \mathbb{P}\big[N(0,1)\geq 1 \big]\geq 2\cdot\frac{1}{10}=\frac{1}{5}\,,
\end{aligned}
\end{equation}
where we also used the fact that $G(x)$ is decreasing for $x \geq 0$, and that $\delta<1$. Now since $\zeta<\frac{1}{2}$, we deduce that
\begin{align*}
    \int_{x \not\in J_1\cup J_2} G_{\delta, \e}^{(n)}(x) \, dx &= \frac{1}{\normo{G_{\delta,\epsilon}} }\int_{x \not\in J_1\cup J_2} G_{\delta, \e}(x) \, dx \leq 10\zeta^{8}\leq\zeta\,.
\end{align*}

Next, we will bound the chi-square divergence
\begin{lemma}
    \label{lem:chi_square}
    Let $A,B$ be defined as above, then $\chi^2\Paren{A, N(0,1)}= O \Paren{\frac{\delta}{\e}}^2$ and $\chi^2\Paren{B, N(0,1)} = O \Paren{\frac{\delta}{\e}}^2$.
\end{lemma}
\begin{proof}
    We will start with $A$:
    \begin{align*}
        \chi^2\Paren{A, N(0,1)} &= \frac{1}{\normo{G_{\delta,\e}}^2} \Paren{\frac{\delta}{2\e}}^2 \sum_{n \in \Z} \int_{n\delta-\e}^{n \delta + \e} \frac{G(x)^2}{G(x)} \, dx - 1\\
        &\leq 25 \cdot \Paren{\frac{\delta}{2\e}}^2 \int_{-\infty}^\infty G(x) \, dx = O \Paren{\frac{\delta}{\e}}^2\,,
    \end{align*}
    where  we used $\normo{G_{\delta,\e}} \geq\frac{1}{5}$ from \eqref{eq:boundGdeTotalMeasure}.

    For $B$ we get:
    \begin{align*}
        \chi^2\Paren{B, N(0,1)} &= \frac{1}{\normo{G_{\delta,\e}}^2} \Paren{\frac{\delta}{2\e}}^2 \Brac{ \sum_{\abs{n} > d} \int_{n\delta-\e}^{n \delta + \e} \frac{G(x)^2}{G(x)} \, dx + \sum_{\abs{n} \leq d} \int_{n\delta-5\e}^{n \delta -3 \e} \frac{G(x+4\e)^2}{G(x)} \, dx }-1 \,.
    \end{align*}
    The term corresponding to the first sum is less than or equal to $\chi^2\Paren{A, N(0,1)} = O \Paren{\frac{\delta}{\e}}^2$, and for the second sum we notice that
    \begin{align*}
        \frac{G(x+4\e)^2}{G(x)} = \frac{1}{\sqrt{2\pi}}\exp\Paren{-(x+4\e)^2 + \frac{x^2}{2}} &= \frac{1}{\sqrt{2\pi}}\exp\Paren{-\frac{x^2}{2} - 8x\e - 16\e^2} \\
        &= \frac{1}{\sqrt{2\pi}}\exp\Paren{16\e^2} \exp\Paren{-\frac{(x+8\e)^2}{2}}
    \end{align*}
    implying that
    \begin{align*}
        \sum_{\abs{n} \leq d} \int_{n\delta-5\e}^{n \delta -3\e} \frac{G(x+4\e)^2}{G(x)} \, dx \leq \int_{-\infty}^\infty \frac{\exp\Paren{16\e^2}}{\sqrt{2\pi}} \exp\Paren{-\frac{(x+8\e)^2}{2}} \, dx = O(1)\,.
    \end{align*}
    Putting everything together yields the claim.
\end{proof}

Lastly, we show that the moments match up to the desired error. We start with $A$.
\begin{lemma}
    \label{lem:moments_A}
Let $k \in \N$.
For the distribution $A$ defined as above and all $t \leq k$ it holds that $$\Abs{\E A^t - \E G^t} \leq O(t!) \cdot \exp(-\Omega(1/\delta^2)) \,.$$
\end{lemma}
\begin{proof}
Since $G_{\delta,\e} = \sum_{n \in \Z} G(x) \cdot f(\frac{x - n\delta}{\delta})$ for $f(y) = \Paren{\frac{\delta}{2\e}} \cdot \ind{y \in [-\e/\delta, \e/\delta]}$ it follows from \cref{fact:generic_moment_bound} that 
$$\abs{\E G^t - \E G_{\delta,\epsilon}^t} \leq t! \cdot \delta^t \cdot \exp(-\Omega(1/\delta^2)) \,.$$

In particular, for $t=0$, we get
$$\abs{1 - \normo{G_{\delta,\e}}}\leq \exp(-\Omega(1/\delta^2))\,,$$
which implies that
$$\Abs{\frac{1}{\normo{G_{\delta,\e}}}-1}\leq \frac{\exp(-\Omega(1/\delta^2))}{1-\exp(-\Omega(1/\delta^2))}\leq O(1)\cdot\exp(-\Omega(1/\delta^2))\,.$$

Now using the fact that $\E G^t=0\leq t!$ for odd $t$ and that 
$$\E G^t =(t-1)!!=(t-1)(t-3)(t-5)\cdots 1\leq t!$$
for even $t$, we get
\begin{align*}
\Abs{\E G_{\delta,\epsilon}^t} \leq \E G^t +  t! \cdot \delta^t \cdot \exp(-\Omega(1/\delta^2))\leq t!\Paren{1+  \delta^t \cdot \exp(-\Omega(1/\delta^2))}\leq O(t!) \,.
\end{align*}

It follows that 
\begin{align*}
\Abs{\E A^t - \E G_{\delta, \e}^t} &=\Abs{\E \Paren{G_{\delta, \e}^{(n)}}^t - \E G_{\delta, \e}^t}=\Abs{\E G_{\delta, \e}^t}\cdot\Abs{\frac{1}{\normo{G_{\delta,\e}}}-1}\leq O(t!)\cdot\exp(-\Omega(1/\delta^2)) \,.
\end{align*}
We conclude that
\begin{align*}
\Abs{\E A^t - \E G^t} &\leq \abs{\E A^t - \E G_{\delta,\epsilon}^t} + \abs{\E G^t - \E G_{\delta,\epsilon}^t}\\
&\leq O(t!)\cdot\exp(-\Omega(1/\delta^2)) + t! \cdot \delta^t \cdot \exp(-\Omega(1/\delta^2))\\
&\leq O(t!)\cdot\exp(-\Omega(1/\delta^2))\,.
\end{align*}
\end{proof}

Next, we will prove the bound for $B$:
\begin{lemma}
    \label{lem:moments_B}
Let $k \in \N$.
For the distribution $B$ defined as above and all $t \leq k$ it holds that $$\abs{\E B^t - \E G^t} \leq O(t!) \cdot \mathrm{exp}(-\Omega(1/\delta^2)) + 4\epsilon\Paren{12 \sqrt{\log(1/\zeta)}}^t .$$
\end{lemma}
\begin{proof}
Due to  \cref{lem:moments_A}, it is sufficient to show that $\abs{\E B^t - E A^t} \leq 4\epsilon\Paren{12 \sqrt{\log(1/\zeta)}}^t $.
    To this end, notice that the two distributions agree for $x$ larger in magnitude than $d\delta - 5\e$.
    Thus
    
    \begin{align*}
        \E B^t - \E A^t &= \int_{-d\delta-5\e}^{d\delta+5\e} x^t dB(x) - \int_{-d\delta+5\e}^{d\delta-5\e} x^t dA(x)  = \int_{-d\delta-5\e}^{d\delta-3\e} (x-4\e)^t dA(x) - \int_{-d\delta+\e}^{d\delta-\e} x^t dA(x) \\         
&= \int_{-d\delta-\e}^{d\delta+\e} (x-4\e)^t dA(x) - \int_{-d\delta+\e}^{d\delta-\e} x^t dA(x)  = \int_{-d\delta-\e}^{d\delta+\e} \big((x-4\e)^t-x^t\big)\cdot dA(x)\,.
    \end{align*}
   Therefore,
   \begin{align*}
   \Abs{ \E B^t - \E A^t}&\leq \sup_{-d\delta-\e\leq x\leq d\delta+\e} \Abs{(x-4\e)^t-x^t}\leq \sup_{-2d\delta\leq x\leq 2d\delta}\Abs{(x-4\e)^t-x^t}\,.
   \end{align*}
   
   Now since
   \begin{align*}
   (x-4\e)^t-x^t &= \sum_{l=1}^t\binom{t}{l}(-4\e)^l x^{t-l}\,,
   \end{align*}
    we get
       \begin{align*}
\Abs{(x-4\e)^t-x^t} &\leq \sum_{l=1}^t\binom{t}{l}(4\e)^l \Abs{x}^{t-l}\leq \sum_{l=1}^t\binom{t}{l}4\e \Paren{1+\Abs{x}}^t\leq 2^t\cdot 4\epsilon \Paren{1+\Abs{x}}^t = 4\epsilon\Paren{2+2\Abs{x}}^t\,.
   \end{align*}
   Now since $d\delta=4\sqrt{\log(1/\zeta)}$, we conclude that
   \begin{align*}
   \Abs{ \E B^t - \E A^t}&\leq 4\epsilon(2+4d\delta)^t=4\epsilon\Paren{2+8 \sqrt{\log(1/\zeta)}}^t  \,.
   \end{align*}
    
\end{proof}

Since $\zeta < 1/2$ and $\log(2) > 1/4$ it holds that $$\sqrt{\log(1/\zeta)} > \sqrt{\log(2)} >  \sqrt{1/4} = 1/2 \,.$$ 
Hence, we obtain $$\Abs{ \E B^t - \E A^t} \leq 4\epsilon\Paren{4 \cdot \frac{1}{2}+8 \sqrt{\log(1/\zeta)}}^t \leq 4\epsilon\Paren{12 \sqrt{\log(1/\zeta)}}^t \,.$$

%% file: content/generic_sq_lower_bound.tex
\section{Discussion of \cref{lem:generic_SQ_lower_bound}}
\label{sec:generic_sq_lower_bound}

As we already mentioned, the proof of \cref{lem:generic_SQ_lower_bound} is verbatim the same as the one of Proposition 3.8 in \cite{DK20}.
The only difference is that we apply their Lemma 3.5 with an arbitrary $c > 0$ instead of $c = 1/2$.
Further, we need the following more precise version of their Fact 3.6.
Also here, the proof is the same (and straightforward) but for us it is important to know the explicity dependence of the size of the set on the parameter $c$. 
In \cite{DK20} it was only stated as $2^{\Omega_c(m)}$.
\begin{fact}
    \label{fact:large_set_random_unit_vectors}
    Let $c > 0$. 
    There exists a set $S$ of unit vectors over $\R^m$ of size $\exp(c^2 \cdot \Omega(m))$ such that for all $u,v \in S$ it holds that $\abs{\iprod{u,v}} \leq c$.
    Further, the constant inside the $\Omega(\cdot)$ notation is independent of $c$.  
\end{fact}
\begin{proof}
    Let the elements of $S$ be picked independently and uniformly at random from the unit sphere.
    Let $u,v \in S$ and w.l.o.g. assume that $u = e_1$.
    Since $v$ has the same distribution as $X/\norm{X}$ where $X \sim N(0,\Id_m)$ it holds that
    $$\Psymb [\iprod{u,v} > c] = \Psymb [\abs{v_1} > c] = \Psymb [\abs{X_1} > c \cdot \norm{X}] \leq \Psymb [\abs{X_1} > c \sqrt{m/2}] + \Psymb [\norm{X} < \sqrt{m/2}] \,.$$
    By standard Gaussian tail bounds the first probability is at most $2\exp(-c^2 m/4)$ and by standard chi-squared tail bounds (e.g., \cite{W19}) the second probability is at most $2\exp(-m/32)$.
    Hence, the claim follows by a union bound over all pairs of distinct elements in $S$.
\end{proof}

%% file: content/moment_bound.tex
\section{Moment Bounds}
\label{sec:moment_bound}
In the following, we prove a statement that is similar to a lemma that was previously shown in \cite{DK20}.

\begin{fact}
    \label{fact:generic_moment_bound}
    Let $f \colon \R \rightarrow \R$ be a non-negative function such that
    \begin{itemize}
    \item $f(x)=0$ for $x\notin[0,1]$, and
    \item $\int_0^1 f(x) \, dx = 1$.
\end{itemize}
Let $G \sim N(0,1)$ and with a little abuse of notation, we will denote the pdf of $G$ also by $G$. For every $\delta > 0$, define $$G_\delta (x) = \sum_{n \in \Z} f\Paren{\frac{x + n \delta}{\delta}} \cdot G(x) \,.$$ We have $$\abs{\E G^t - \E G_\delta^t} \leq t! \cdot \delta^t \cdot \mathrm{exp}(-\Omega(1/\delta^2)) \,.$$
\end{fact}

The proof will make use of Fourier analysis.
We will introduce here the necessary background.
For a function $g \colon \R \rightarrow \R$ we define its Fourier transform to be $$\hat{g}(\omega) = \frac{1}{\sqrt{2\pi}} \int_{-\infty}^\infty g(x) \cdot e^{-i  \omega x} \, dx \,.$$
It is well-known that for $a, b \in \R$ and $h \colon \R \rightarrow \R$ we have
\begin{align*}
   \widehat{(a \cdot g + b \cdot h)} = a \cdot \hat{g} + b \cdot \hat{h} \quad\text{and} \quad \widehat{(g \cdot h)} = \frac{1}{\sqrt{2\pi}} \Paren{\hat{g} * \hat{h}}\,,
\end{align*}
where $*$ denotes convolution.
Further, if $G$ denotes the pdf of a standard Gaussian then $\hat{G} =G$.

For a random variable $X$ with pdf $g$, let $$\varphi_X(t) = \int_{-\infty}^\infty g(x) \cdot e^{itx} \, dx$$ denote its characteristic function.
Notice that $\varphi_X(t) = \sqrt{2 \pi} \cdot \hat{g}(-t)$.
For the $t$-th moment of $X$ it follows
\begin{align*}
\E X^t = \frac{1}{i^t} \cdot \varphi_X^{(t)}(0) &= \sqrt{2\pi} \cdot (-i)^t \cdot (-1)^t\cdot \hat{g}^{(t)}(0) \\
&= \sqrt{2\pi} \cdot i^t \cdot \hat{g}^{(t)}(0) \,.
\end{align*}

\begin{proof}[Proof of \cref{fact:generic_moment_bound}]
    By the above discussion it is enough to show that $$\abs{\hat{G}_\delta^{(t)} (0) - \hat{G}^{(t)}(0)} \leq t! \cdot \frac{\delta^t}{\sqrt{2\pi}} \cdot \exp^{-\Omega((1/\delta)^2)} \,.$$
    We know that $\hat{G} =  G$.
    Next, we will compute $\hat{G}_\delta (\omega)$.
    Let
    $$F(x) = \sum_{n \in \Z} f\left(\frac{x + n \delta}{\delta}\right)\,,$$
    so that $G_\delta = G(x) \cdot F(x)$ and hence
    $$\hat{G_\delta} = \frac{1}{\sqrt{2\pi}}\hat{G}(\omega) * \hat{F}(\omega)\,.$$
    For $F$ we obtain the following:
    Since $F$ is periodic with period $\delta$, we can decompose it using the Fourier basis over $[0,\delta]$.
    We get that
    \begin{align*}
    F(x) = \sum_{n \in \Z} \hat{F}_n \cdot e^{2 i \pi \frac{n}{\delta}x}\,,
    \end{align*}
    where
    \begin{align*}
    \hat{F}_n &= \frac{1}{\delta} \int_0^\delta F(x) \cdot e^{ - 2 i \pi \frac{n}{\delta}x} \, dx = \frac{1}{\delta} \int_0^\delta \sum_{l \in \Z} f\Paren{\frac{x + l \delta}{\delta}} \cdot e^{ - 2 i \pi \frac{n}{\delta}x} \, dx = \frac{1}{\delta} \int_0^\delta f\Paren{\frac{x}{\delta}} \cdot e^{ - 2 i \pi \frac{n}{\delta}x} \, dx \\
    &= \int_0^1 f(y) \cdot e^{ - 2 i \pi n y} \, dy\,.
    \end{align*}
    
    For $n=0$, we clearly have $\hat{F}_0 = 1$ since $f$ integrates to 1 over $[0,1]$.
    Now for $n \neq 0$, we can write $$\abs{\hat{F}_n} = \Abs{\int_0^1 f(y) \cdot e^{ - 2 i \pi n y} \, dy} \leq \int_0^1 \abs{f(y) \cdot e^{ - 2 i \pi n y}} \, dy = \int_0^1 f(y) \, dy = 1\,, $$
    where we used the fact that $f$ is non-negative and that the complex exponential has magnitude 1.
    
    Now using the fact that the Fourier transform of $e^{2 i \pi \frac{n}{\delta}x}$ is equal to $\sqrt{2\pi}\bm \delta_{\mathrm{D}} \Paren{\omega - \frac{2\pi n}{\delta}}$, where $\bm \delta_{\mathrm{D}}$ is the Dirac delta-distribution\footnote{Note that we use bold font and the subscript $\mathrm{D}$, i.e., $\bm \delta_{\mathrm{D}}$, to denote the Dirac delta-distribution and non-bold font for the paramter $\delta \in \R$.}, we get that
    $$\hat{F}(\omega) = \sqrt{2\pi}\cdot \sum_{n \in \Z} \hat{F}_n \cdot \bm \delta_{\mathrm{D}} \Paren{\omega - \frac{2\pi n}{\delta}}\,, $$
     and hence $$\hat{G}_\delta(\omega) = \frac{1}{\sqrt{2\pi}}\hat{G} (\omega) * \hat{F}(\omega) = \sum_{n \in \Z} \hat{F}_n \cdot \hat{G} (\omega) * \bm \delta_{\mathrm{D}} \Paren{\omega - \frac{2\pi n}{\delta}}=\sum_{n\in Z} \hat{F}_n \cdot \hat{G}\Paren{\omega - \frac{2\pi n}{\delta}} \,.$$
Now since $\hat{F}_0=1$ and $|\hat{F}_n|\leq 1$ for $n\neq 1$, we get $$\hat{G}_\delta^{(t)} (0) = \hat{G}^{(t)}(0) + \sum_{n \neq 0} \hat{F}_n \cdot \hat{G}^{(t)}\Paren{ - \frac{2\pi n}{\delta}}\,,$$
    and 
\begin{align*}
\abs{\hat{G}_\delta^{(t)} (0) - \hat{G}^{(t)}(0)} &\leq \sum_{n \neq 0}~\left|\hat{G}^{(t)} \Paren{ \frac{2\pi n}{\delta}}\right| \,.
\end{align*}

Now using Cauchy's integral formula, we have
$$\hat{G}^{(t)} \Paren{ \frac{2\pi n}{\delta}} = \frac{t!}{2\pi i}\oint_{\gamma_n}\frac{\hat{G}(z)}{\Paren{z-\frac{2\pi n}{\delta}}^{t+1}}dz\,,$$
where the complex (contour) integral is over the circle $\gamma_n$ of center $\frac{2\pi n}{\delta}$ and of radius $\frac{\pi }{2\delta}$ in the complex plane. Now since the circle $\gamma_n$ has length $\frac{\pi^2}{\delta}$, we get
\begin{align*}
\left|\hat{G}^{(t)} \Paren{ \frac{2\pi n}{\delta}}\right| &= \frac{t!}{2\pi }\left|\oint_{\gamma_n}\frac{\hat{G}(z)}{\Paren{z-\frac{2\pi n}{\delta}}^{t+1}}dz\right| \leq \frac{t!}{2\pi }\cdot \frac{\pi^2}{\delta}\cdot\max_{z\in\gamma} ~\left|\frac{\hat{G}(z)}{\Paren{z-\frac{2\pi n}{\delta}}^{t+1}}\right|\\
&= t!\cdot\frac{\pi}{2\delta}\cdot\frac{\max_{z\in\gamma}|\hat{G}(z)|}{\Paren{\frac{\pi}{2\delta}}^{t+1}} \leq t! \cdot \frac{\delta^t}{\sqrt{2\pi}}\cdot e^{-\Omega((n/\delta)^2)}\,,
\end{align*}
where in the last inequality we used the fact that
\begin{align*}
\max_{z\in\gamma_n}|\hat{G}(z)| &= \max_{z\in\gamma_n}|G(z)| = \frac{1}{\sqrt{2\pi}}\max_{z\in\gamma_n}~\left|e^{-\frac{z^2}2}\right| = \frac{1}{\sqrt{2\pi}}\max_{x+iy\in\gamma_n}~e^{-\frac{x^2-y^2}2}\\
&\leq  \frac{1}{\sqrt{2\pi}}e^{-\frac{1}{2}\Paren{\Paren{\frac{2\pi |n|}{\delta}-\frac{\pi}{2\delta}}^2-\Paren{\frac{\pi}{2\delta}}^2}}\leq \frac{1}{\sqrt{2\pi}} e^{-\frac{1}{2}\Paren{\frac{2\pi |n|}{\delta}-\frac{\pi}{\delta}}^2}\leq \frac{1}{\sqrt{2\pi}}e^{-\frac{1}{2}\Paren{\frac{\pi |n|}{\delta}}^2}=\frac{1}{\sqrt{2\pi}} e^{-\Omega((n/\delta)^2)}\,.
\end{align*}

We conclude that

\begin{align*}
\abs{\hat{G}_\delta^{(t)} (0) - \hat{G}^{(t)}(0)} &\leq t! \cdot \frac{\delta^t}{\sqrt{2\pi}}\cdot \sum_{n \neq 0}  e^{-\Omega((n/\delta)^2)}= t! \cdot \frac{\delta^t}{\sqrt{2\pi}}\cdot e^{-\Omega((1/\delta)^2)}\,.
\end{align*}
    \end{proof}